\documentclass[runningheads]{llncs}
\usepackage{adjustbox}

\usepackage{graphicx}
\usepackage{mathtools}
\usepackage{amsfonts}
\usepackage{amsmath}
\usepackage{amssymb}
\usepackage{enumerate}
\usepackage{tikz-cd}
\usepackage{bussproofs}
\usepackage{xcolor}
\usepackage[colorlinks,linkcolor=blue]{hyperref}
\usepackage{enumitem}[shortlabels]
\EnableBpAbbreviations

\newcommand{\val}[1]{[\![{#1}]\!]}
\newcommand{\descr}[1]{(\![{#1}]\!)}
\newcommand{\nc}{\,\mid\!\sim}

\newcommand{\marginnote}[1]{\marginpar{\raggedright\tiny{#1}}}

\newcommand{\fakeparagraph}[1]{

\textit{#1} \ \ }
\usepackage[T1]{fontenc}
\begin{document}
\title{Defeasible Reasoning on Concepts\thanks{This paper is partially  funded by the EU MSCA (grant No.~101007627).}}
%
%
\author{Yiwen Ding\inst{1}\and
Krishna Manoorkar \inst{1}\orcidID{0000-0003-3664-7757} \and
Ni Wayan Switrayni\inst{1}\orcidID{0000-0001-9560-5739}\and Ruoding Wang\inst{1,2}\orcidID{0009-0005-3995-3225}
}
\authorrunning{Ding  et al.}
%
\institute{Vrije Universiteit Amsterdam, The Netherlands \and Xiamen University, China
\\
\email{\{dyiwen666\}@gmail.com}}
\maketitle              
\begin{abstract} 
In this paper, we take first steps toward developing defeasible reasoning on concepts in KLM framework. We define generalizations of cumulative reasoning system $\mathbf{C}$ and cumulative reasoning system with loop $\mathbf{CL}$ to conceptual setting. We also generalize cumulative models, cumulative ordered models, and preferential models to conceptual setting and show the soundness and completeness results for these models.
\keywords{Non-monotonic reasoning \and Cumulative model \and Polarity-based semantics \and Formal concept analysis}
\end{abstract}
\section{Introduction} \marginnote{KM: too long for now}
  Formal Concept Analysis (FCA) is a mathematical tool, as developed in \cite{ganter2012formal}, commonly used in Knowledge Representation and Reasoning to study conceptual hierarchies. FCA has applications across various fields, including information retrieval, association rule mining, data analysis, and ontology engineering. Lattice-based propositional logic, along with its polarity-based semantics, has been developed as a logic for reasoning about formal contexts and the concepts they define \cite{10.1007/978-3-662-52921-8_10,Conradie_2017}. This logic establishes a monotone consequence relation, denoted by $\vdash$, between concepts. Specifically, $C_1 \vdash C_2$ is interpreted as `all the objects in $C_1$ are in $C_2$', or equivalently, `all the features in the description of $C_2$ are  in the description of  $C_1$', which means that `$C_1$ is a subconcept of $C_2$'.

However, in many real-life applications, it is crucial to define a defeasible consequence relation, denoted by $\nc_A$ (or $\nc_X$), which formalizes the notion that `all the objects in $C_1$ are in $C_2$, with some exceptions' (or `all the features in the description of $C_2$ are in the description of $C_1$, with some exceptions'). In other words, this relation captures the idea that all the `typical' or `normal' objects (resp.~features) in $C_1$ (resp.~$C_2$) are in $C_2$ (resp.~$C_1$)\footnote{In this paper, we focus exclusively on the defeasible consequence relation $\nc_A$ (denoted by $\nc$), which pertains to typical objects, due to space constraints. Exploring $\nc_X$ and its interaction with $\nc_A$ would be an interesting direction for future research.}. It is important to note that such a relation, $\nc_A$ (or $\nc_X$), is usually non-monotonic. For example, let $C_1$ and $C_2$ represent the concepts of `mammals' and `viviparous animals', respectively. Since most mammals are typically viviparous, we have $C_1 \nc_A C_2$. However, if we introduce $C_3$, representing the concept of `echidnas', which are a kind of oviparous mammal, we find that $C_3 \vdash C_1$ (i.e., all echidnas are mammals), but $C_3 \not\nc_A C_2$ (i.e., typically, echidnas are not viviparous).

To formalize the relation $\nc$, we employ the framework developed by Kraus, Lehmann, and Magidor (commonly referred to as the KLM framework) \cite{kraus1990nonmonotonic}. We define the reasoning systems $\mathbf{CC}$ and $\mathbf{CCL}$ as the conceptual counterparts of the cumulative reasoning systems $\mathbf{C}$ and $\mathbf{CL}$, respectively, as defined in \cite{kraus1990nonmonotonic}. Since the language of lattice-based propositional logic is not closed under $\rightarrow$ and $\neg$, we cannot directly apply the framework from \cite{kraus1990nonmonotonic}. Nonetheless, we show that the KLM framework can be extended to reason about concepts with suitable modifications. We further generalize cumulative models and cumulative ordered models to conceptual cumulative models and conceptual cumulative ordered models, which are sound and complete with respect to $\mathbf{CC}$ and $\mathbf{CCL}$, respectively. Additionally, we define the conceptual counterparts of preferential models and show that, unlike in the setting of \cite{kraus1990nonmonotonic}, conceptual preferential models are complete with respect to $\mathbf{CC}$.

\fakeparagraph{Structure of the paper.}
In Section \ref{sec:prelim}, we provide the necessary preliminaries on the KLM framework for non-monotonic reasoning and lattice-based logic as the logic for concepts. In Section \ref{sec:KLM framework for reasoning on concepts}, we generalize the defeasible reasoning systems $\mathbf{CC}$ and $\mathbf{CCL}$ to the conceptual setting. We also define the conceptual counterparts of cumulative models, cumulative ordered models, and preferential models, and discuss the soundness and completeness proofs for them. In Section \ref{sec:Example}, we give an example to demonstrate non-monotonic reasoning on concepts. In Section \ref{sec:Conclusion and future works}, we conclude and give directions for future research.

\section{Preliminaries}\label{sec:prelim} 
In this section, we gather some useful preliminaries about the KLM framework for defeasible reasoning developed in 
\cite{kraus1990nonmonotonic}, and lattice-based propositional logic and its polarity-based semantics based on \cite{conradie2021nondistributivelogicssemanticsmeaning} and \cite{Conradie_2017}. For a detailed discussion, we refer to \cite[Section 2]{Blackburn_Rijke_Venema_2001}. 

\subsection{KLM Framework for Defeasible Reasoning}\label{ssec:classical cumulative reasoning}
The language $L$ of cumulative logic is defined over a set of propositional variables consisting propositional connectives $\neg, \vee, \wedge, \rightarrow,$ and $\leftrightarrow$. Negation and disjunction are considered as the primitive connectives and the rest as defined connectives.
Hence, $L$ can be considered as the set of all propositional formulas. 

A cumulative logical system \textbf{C} on $L$ consists of cumulative $L$-sequents $\phi \nc \psi$ (read as \textit{if $\phi$, normally $\psi$}, or \textit{$\psi$ is a plausible consequence of $\phi$}) containing the Reflexivity axiom $\phi \nc \phi$ and closed under the following inference rules:

{
\centering
\par 
\begin{tabular}{rlrl}
  \scriptsize{Left Logical Equivalence (LLE)}  & $\frac{\phi \leftrightarrow \psi \quad \phi \nc \chi}{\psi \nc \chi}$  &  $\frac{\phi\rightarrow \psi\quad \chi\nc\phi}{\chi\nc\psi}$
     & \scriptsize{Right Weakening (RW)} \\
    \scriptsize{Cautious Monotonicity (CM)} & $\frac{\phi \nc\psi \quad \phi\nc\chi}{\phi \wedge \psi \nc\chi}$ &$\frac{\phi\wedge\psi\nc\chi\quad\phi\nc\psi}{\phi\nc\chi}$ & \scriptsize{(Cut)}.
\end{tabular}

}

\noindent Such plausible consequence relation $\nc$ is called  {\em cumulative consequence relation}. 

\noindent \fakeparagraph {Cumulative models.} Now, we discuss the semantics for cumulative reasoning, i.e.,~for system $\textbf{C}$. Let $\mathcal{U}$ be a set 
of all worlds that the reasoner considers possible. The satisfaction relation
between worlds and formulas behaves as expected with regard to propositional connectives. Given $u\in \mathcal{U}$, and $\phi,\psi \in L$, we write $u\models \phi$ if $u$ satisfies $\phi$, $u\models \neg \phi$ iff $u\not\models \phi$, and $u\models \phi\vee\psi$ iff $u\models \phi$ or $u\models \psi$. 

Let $P\subseteq U$ for some set $U$ and $\prec$ a binary relation on $U$. We say that $t\in P$ is {\em minimal} in $P$ iff $\forall s\in P, s\not\prec t$. We say that $t\in P$ is a {\em minimum} of $P$ iff $\forall s\in P$ such that $s\not = t, t\prec s$. Furthermore, $P$ is {\em smooth} iff $\forall t \in P$, either there exists $s$ minimal in $P$ such that $s \prec t$, or $t$ itself is minimal in $P$.

Let $\mathcal{W}=(S,l,\prec)$ be such that $S$ is a set of elements called {\em states}, $l:S\rightarrow \mathcal{P}(\mathcal{U)}$ is a function that labels every state with a non-empty set of worlds, and $\prec$ is a binary relation on $S$.
The {\em satisfaction} relation $\models \subseteq S \times L$ on $\mathcal{W}$ is defined as follows: For any $\phi \in L$, and $s \in S$, $s \models \phi$ iff for all $u \in l(s)$, $u \models \phi$. $\mathcal{W}$ is said to be a {\em cumulative model} if the set $\widehat{\phi}=\{s\mid s\in S, s\models \phi\}$ is smooth for any $\phi \in L$. Any cumulative model $\mathcal{W}$ defines a {\em consequence relation $\nc_\mathcal{W}$} 
as follows: For any $\phi, \psi \in L$, 
 $\phi \nc_\mathcal{W}  \psi$ iff for any $s$ minimal in  $\widehat{\phi}$, we have $s \in \widehat{\psi}$.

Finally, the following theorem gives the soundness and completeness of  system $\textbf{C}$ w.r.t.~cumulative models \cite[Theorem 1]{kraus1990nonmonotonic}.
\begin{theorem}[Representation theorem for cumulative relations]
A consequence relation is a cumulative consequence relation iff it is defined by some cumulative model.
\end{theorem}

\fakeparagraph{Cumulative ordered models.}A cumulative model is said to be a {\em cumulative ordered model} if the relation $\prec$ is a strict partial order. 
It is proven (cf.~\cite[Theorem 2]{kraus1990nonmonotonic}) that a cumulative consequence relation validates the following rule (Loop) for all $n$ iff it is represented by some cumulative ordered model. 

{\small
\[\frac{\phi_0\nc\phi_1\quad \phi_1\nc\phi_2\quad\dots\quad\phi_{n-1}\nc\phi_n\quad\phi_n\nc\phi_0}{\phi_0\nc\phi_n}\quad\text{\footnotesize (Loop).}
\]
}
\noindent We call the extension of $\mathbf{C}$ with the rule (Loop) $\mathbf{CL}$. 

\fakeparagraph{Preferential models.}A cumulative ordered model is said to be a {\em preferential model} if the labelling function $l$ assigns every state a single possible world, and $\prec$ is a strict partial order. 
It is proven (cf.~\cite[Theorem 3]{kraus1990nonmonotonic}) that a cumulative consequence relation validates the following rule (Or) iff it is represented by some preferential model. 
{\small
\[\frac{\phi \nc \chi \quad \psi \nc \chi  }{\phi \vee \psi \nc \chi}\quad\text{\footnotesize (Or).}
\]
}

\subsection{Formal Concept Analysis and Lattice-based Propositional Logic}\label{ssec:Non-distributive modal logic}
In this section, we gather preliminaries on Formal Concept Analysis and lattice-based propositional logic as a reasoning system for it (see \cite{10.1007/978-3-662-52921-8_10,Conradie_2017}, for more details).  

A {\em formal context} or {\em polarity} is a tuple $\mathbb{P}=(A,X,I)$ such that $A$ and $X$ are sets interpreted as sets of {\em objects} and {\em features}, respectively, and the relation $I\subseteq A\times X$ is interpreted as $a I x$ if `object $a$ has feature $x$'. The maps $(\cdot)^\uparrow :\mathcal{P}(A)\rightarrow\mathcal{P}(X)$ and $(\cdot)^\downarrow :\mathcal{P}(X)\rightarrow\mathcal{P}(A)$, defined as $B^\uparrow := \{x \in X \mid \forall b \in B, b I x\}$ and $Y^\downarrow := \{a \in A \mid \forall y \in Y, a I y\}$, where $B\subseteq A$ and $Y\subseteq X$, form a {\em Galois connection} between posets $(\mathcal{P}(A),\subseteq)$ and $(\mathcal{P}(X),\subseteq)$, that is, $Y\subseteq B^\uparrow$ iff $B\subseteq Y^\downarrow$ for all $B\in \mathcal{P}(A)$ and $Y\in \mathcal{P}(X)$.
A {\em formal concept} or {\em category} of $\mathbb{P}$ is a pair $c=(\val{c},\descr{c})$ such that $\val{c}\subseteq A$, $\descr{c}\subseteq X$, and $\val{c}^\uparrow = \descr{c}$, $\descr{c}^\downarrow = \val{c}$. It follows that $\val{c}$ and $\descr{c}$ are {\em Galois-stable}, i.e.\ $\val{c}^{\uparrow\downarrow} = \val{c}$ and $\descr{c}^{\downarrow\uparrow} = \descr{c}$. The set of all formal concepts of $\mathbb{P}$ can be partially ordered as follows: For any formal concepts $c$ and $d$, $c\leq d$ iff $\val{c}\subseteq \val{d}$ iff $\descr{d}\subseteq \descr{c}$.
This poset $\mathbb{P}^+$ is a complete lattice where meet and join are given by $\bigwedge \mathcal{H}:=(\bigcap_{c\in \mathcal{H}} \val{c}, (\bigcap_{c\in \mathcal{H}} \val{c})^\uparrow)$ and $\bigvee \mathcal{H}:=((\bigcap_{c\in \mathcal{H}} \descr{c})^\downarrow,\bigcap_{c\in \mathcal{H}} \descr{c})$ for any $\mathcal{H}\subseteq \mathbb{P}^+$. It is then called the {\em concept lattice} of $\mathbb{P}$.

We define the {\em lattice-based propositional logic} $\mathbf{L}$ for reasoning about concepts as follows. Let $\mathsf{Prop}$ be a countable set of propositional variables. The language $\mathcal{L}$ (i.e. set of formulas) of $\mathbf{L}$ is defined by the following recursion:

{{
\centering
$\phi ::= p \mid  \bot \mid \top \mid \phi \wedge \phi \mid \phi \vee \phi$,
\par
}}
   
\noindent
where $p\in \mathsf{Prop}$. 
$\mathbf{L}$ is the same as the smallest logic containing the axioms:

{{\centering 
$p \vdash p,
\ \
p \vdash \top
\ \
\bot \vdash p,
\ \
p \vdash p \vee q,
\ \
q \vdash p \vee q,
\ \
p \wedge q \vdash p,
\ \
p \wedge q \vdash q,
\ $
\par}}

\noindent and closed under the following inference rules:

{{
\centering
$\frac{\phi\vdash \chi\quad \chi\vdash \psi}{\phi\vdash \psi}
\ \ 
\frac{\phi\vdash \psi}{\phi\left(\chi/p\right)\vdash\psi\left(\chi/p\right)}
\ \ 
\frac{\chi\vdash\phi\quad \chi\vdash\psi}{\chi\vdash \phi\wedge\psi}
\ \ 
\frac{\phi\vdash\chi\quad \psi\vdash\chi}{\phi\vee\psi\vdash\chi}$.
\par
\smallskip
}}

A {\em polarity-based model} is a pair $ \mathbb{M}= (\mathbb{P}, V)$, where $\mathbb{P}$ is a polarity, and $V: \mathsf{Prop} \rightarrow\mathbb{P}^+$ is a {\em valuation} that assigns a concept to each propositional variable. For each $p\in\mathsf{Prop}$, we let $\val{p}: = \val{V(p)}$ (resp.~$\descr{p}: = \descr{V(p)}$) denote the {\em extension} (resp.~{\em intension}) of the {\em interpretation} of $p$ under $V$. A valuation can be homomorphically extended to a unique map $\overline V: \mathcal{L} \rightarrow\mathbb{P}^+$ on all the $\mathcal{L}$-formulas. The connectives $\vee$ and $\wedge$ are given by join and meet of concepts as discussed above, which are interpreted as the least common super-concept and the greatest common sub-concept, respectively, while $\bot$ and $\top$ represent the smallest and the largest\footnote{We do not include negation in the language as there is no notion of negation of concepts accepted in FCA community in general.}.

Given a polarity-based model $\mathbb{M}$, the {\em satisfaction} relation $\Vdash$ and {\em co-satisfaction} relation $\succ$ are defined inductively as follows: For any $a \in A$, $x \in X$, and $\phi, \psi \in \mathcal{L}$,

{
\centering
\par
 $ \mathbb{M}, a \Vdash \phi$ iff $a \in \val{\overline{V}(\phi)}$, \quad  \quad  $ \mathbb{M}, x \succ \phi$ iff $x \in \descr{\overline{V}(\phi)}$, 
 
}
{
\centering
\par
 $ \mathbb{M} \models \phi \vdash \psi$ iff $\val{\phi} \subseteq \val{\psi}$  iff $\descr{\psi} \subseteq \descr {\phi}$.
 
}
 Note that, for any object $a \in A$  (resp.~feature $x \in X$), and formula $\phi \in \mathcal{L}$, $M, a \Vdash \phi$  (resp.~$M, x \succ \phi$)  is intuitively  interpreted  as `object $a$ is in concept $\phi$' (resp.~`feature $x$ describes concept $\phi$'). The $\mathcal{L}$-sequent $\phi \vdash \psi$ is intuitively interpreted as `every object of $\phi$ is in $\psi$' or `every feature in description of $\psi$ is in $\phi$'.
 Thus, polarity-based models provide a natural system for reasoning about concepts. The logic $\mathbf{L}$ is the set of $\mathcal{L}$-sequents valid on all polarity-based models.

Based on the general theory of lattice-based propositional logic, we give the following version of compactness for it, which would be useful later. 

\begin{proposition}\label{prop:compactness}
  Let $\Gamma \cup \{\phi_0\}$ be a set of  $\mathcal{L}$-formulas. Suppose for any finite $\Gamma' \subseteq \Gamma$, there exists a polarity-based model $\mathbb{M}'=(\mathbb{P}', V')$, where $\mathbb{P}'=(A',X',I')$ and $a' \in A'$ (resp.~$x' \in X'$) such that $\mathbb{M}',a'  \Vdash \gamma'$ (resp.~$\mathbb{M}',x'  \succ \gamma'$)  for any $\gamma' \in \Gamma'$, and $\mathbb{M}',a' \not \Vdash \phi_0$ (resp.~$\mathbb{M}',x' \not \succ \phi_0$). Then, there exists a model $\mathbb{M}=(\mathbb{P}, V)$, where $\mathbb{P}=(A,X,I)$ and $a \in A$ (resp.~$x \in X$) such that $\mathbb{M},a  \Vdash \gamma$ (resp.~$\mathbb{M},x \succ \gamma$) for any $\gamma \in \Gamma$, and $\mathbb{M},a \not \Vdash \phi_0$ (resp.~$\mathbb{M},x \not \succ \phi_0$).
\end{proposition}

The following proposition follows from the fact that both $\vee$ and $\wedge$ are defined in terms of intersections on the concept lattices.  

\begin{proposition}\label{prop:joint dissatiscfaction}
Let $\Gamma$ and $\Delta$ be sets of $\mathcal{L}$-formulas such that for any $\psi \in \Delta$, there exists a polarity-based model $\mathbb{M}_\psi=(\mathbb{P}', V')$, where $\mathbb{P}'=(A',X',I')$ and $a' \in A'$ (resp.~$x' \in X'$), such that $\mathbb{M}_\psi,a' \not \Vdash \psi$ (resp.~$\mathbb{M}_\psi,x' \not \succ \psi$) and $\mathbb{M}_\psi,a'  \Vdash \phi$ (resp.~$\mathbb{M}_\psi,x'  \succ \phi$)  for any $\phi \in \Gamma$. Then, there exists a model  $\mathbb{M}=(\mathbb{P}, V)$, where $\mathbb{P}=(A,X,I)$ and $a \in A$ (resp.~$x \in X$), such that $\mathbb{M},a \not \Vdash \psi$ (resp.~$\mathbb{M},x \not \succ \psi$) and $\mathbb{M},a  \Vdash \phi$ (resp.~$\mathbb{M},x  \succ \phi$)  for any $\phi \in \Gamma$, $\psi \in \Delta$.
\end{proposition}
Note that the propositional counterpart of this proposition is not true. Indeed, if we take $\Gamma =\emptyset$, and $\Delta=\{p,\neg p\}$, then we can have two different models (i.e.,~valuations) which do not validate $p$ and $\neg p$, respectively, but no model that invalidates both of them simultaneously (as $p \vee \neg p$ is a tautology).

\section{KLM Framework for Reasoning on Concepts}\label{sec:KLM framework for reasoning on concepts}
 In this section, we generalize the reasoning system $\mathbf{C}$ and $\mathbf{CL}$ discussed in Section \ref{ssec:classical cumulative reasoning} to conceptual cumulative reasoning  $\mathbf{CC}$ and conceptual cumulative reasoning with loop $\mathbf{CCL}$. We also generalize the cumulative models, cumulative ordered models and preferential models to conceptual settings and   show the soundness and completeness for them. 

To generalize the cumulative reasoning to conceptual setting, we have to make the following modifications to the logic and models
described in \cite{kraus1990nonmonotonic}: 
(1) In \cite{kraus1990nonmonotonic}, the language of underlying logic is assumed to be closed under all the classical connectives including negation and implication. However, lattice-based propositional logic does not have negation and implication in its language. Thus, we replace the formula $\phi \rightarrow \psi$ in the rules and axioms of $\mathbf{C}$ with the sequent $ \phi \vdash \psi$. We choose $\phi \vdash \psi$ as the replacement because it has similar interpretation to $\phi \rightarrow \psi$ in the sense that, for any polarity-based model $\mathbb{M}=(\mathbb{P},V)$ with $\mathbb{P} =(A,X,I)$, $\mathbb{M}\models \phi \vdash \psi $ iff for any $a \in A$, $\mathbb{M},a \Vdash \phi$ implies $\mathbb{M},a \Vdash \psi$.
(2) As the underlying logic in \cite{kraus1990nonmonotonic} is assumed to have implication and deduction theorem, the {\em compactness} of this logic is enough to prove \cite[Lemma 8]{kraus1990nonmonotonic} which is central in completeness proof. In our setting, as we do not have implication in the language, we need the modified version of compactness (cf.~Proposition \ref{prop:compactness}) to obtain the counterpart of that (cf.~Lemma \ref{lem: normal cummulative}). 
(3) In polarity-based models, satisfaction and co-satisfaction relations are defined locally at each object or feature. Thus, possible worlds in our setting are pointed polarity-based models (cf.~Definition \ref{def:pointed polarity-based model}). This is similar to approach used in \cite{BRITZ201155}
to define KLM-style modal logics. 

Given any $\phi,\psi\in\mathcal{L}$, $\phi\nc\psi$ is a {\em cumulative $\mathcal{L}$-sequent}. Similar to the propositional setting, we interpret $C_1 \nc C_2$ as `typically or commonly, objects in $C_1$ are in $C_2$'. 
A {\em lattice-based cumulative logic} is a set of $\mathcal{L}$-sequents closed under all the axioms and rules of lattice-based propositional logic, and cumulative $\mathcal{L}$-sequents closed under the Reflexivity axiom $\phi\nc\phi$, and the following rules:

{{
\centering
\par
\begin{tabular}{rlrl}
\scriptsize{Left Logical Equivalence (LLE)} &
$\frac{\phi\vdash\psi\quad \psi\vdash \phi\quad \phi\nc\chi}{\psi\nc\chi}$  &  $\frac{\phi\vdash\psi\quad\chi\nc\phi}{\chi\nc\psi}$ & \scriptsize{Right Weakening (RW)}\\ 
\scriptsize{Cautious Monotonicity (CM)}&$\frac{\phi\nc\psi\quad\phi\nc\chi}{\phi\wedge\psi\nc\chi}$ &
$\frac{\phi\wedge\psi\nc\chi\quad\phi\nc\psi}{\phi\nc\chi}$ &\scriptsize{(Cut)}.\\
\end{tabular}

}}
\smallskip

\noindent This relation $\nc$ is called \textit{conceptual cumulative consequence relation}.



%

\subsection{Conceptual Cumulative Models}\label{ssuc:conceptual cumulative models}

Now, we introduce conceptual cumulative models to capture cumulative reasoning about concepts from a semantic perspective. Informally, our models consist of states which are sets of pointed polarity-based models (cf.~Definition \ref{def:pointed polarity-based model}), with a binary relation between those states. This relation represents the preferences that the reasoner may have between different states. 
The reasoner, described by a conceptual cumulative model, accepts a conditional assertion $\phi\nc\psi$ if and only if 
 all the pointed polarity-based models in any most preferred states for $\phi$, are also pointed polarity-based models 
 for $\psi$.

\begin{definition}\label{def:pointed polarity-based model}
A {\em pointed polarity-based model} is a tuple $\mathbb{M}_a=(\mathbb{P}, V, a)$, where $\mathbb{P}=(A,X,I)$ is a polarity, $V:\mathrm{Prop} \to \mathbb{P}^+$ is a valuation on $\mathbb{P}$, and $a \in A$. We call $a$ the {\em pointed object} of $\mathbb{M}_a$. 
\end{definition}

\begin{definition}
Let $\mathcal{U}$ be  a set of pointed polarity-based models, and  $\mathcal{M}=(S,l,\prec)$ be a tuple, where  $S$ is a non-empty set of {\em states}, $l:S \to \mathcal{P}(\mathcal{U})$ is a map which assigns each state to a set of pointed polarity-based models, and $\prec$ is a binary relation on $S$. For any $\phi \in \mathcal{L}$, and $s \in S$, $s \models \phi$ iff for all $\mathbb{M}_a\in l(s)$, $\mathbb{M}_a \Vdash \phi$. $\mathcal{M}$ is said to be a {\em conceptual cumulative model} if the set $\widehat{\phi}=\{s\mid s\in S, s\models \phi\}$ is smooth for any $\phi \in \mathcal{L}$. A conceptual cumulative model $\mathcal{M}=(S,l,\prec)$ is called a {\em strong conceptual cumulative model} if the relation $\prec$ is asymmetric and the set $\widehat{\phi}$ has a minimum for every $\phi \in \mathcal{L}$. 
\end{definition}

The relation $\prec$ represents the reasoner’s preference among states. Given two states $s$ and $t$, $s\prec t$ means that, in the reasoner’s mind, $s$ is preferred to or more natural than $t$. For example, when considering the category of birds, one may prefer a state consisting of pointed models with a pigeon as the pointed object over a state consisting of pointed models with a penguin as the pointed object.



 We now define the consequence relation on the conceptual cumulative models.

 \begin{definition}
    Given a conceptual cumulative model $\mathcal{M}=(S,l, \prec)$, the {\em consequence relation} defined by $\mathcal{M}$, denoted as $\nc_\mathcal{M}$, is defined by: $\phi_1 \nc_\mathcal{M}  \phi_2$ iff for any $s$ minimal in  $\widehat{\phi_1}$, we have $s \in \widehat{\phi_2}$.
 \end{definition}

 \subsection{Characterization of Conceptual Cumulative Consequence Relations}\label{ssec:Characterization  cumulative consequence}
 In this section, we shall characterize the relationship between conceptual cumulative models and conceptual cumulative consequence relations. The proof broadly follows the strategy for characterizing the relationship between cumulative models and cumulative consequence relations given in \cite[Section 3.5]{kraus1990nonmonotonic}.
 The following lemma is crucial in proving that $\mathbf{CC}$ is sound w.r.t.~conceptual cumulative models. 
 \begin{lemma}\label{lem:intersection}
For any formulas $\phi, \psi \in \mathcal{L}$,  $\widehat{\phi\wedge\psi} = \widehat{\phi} \cap \widehat{\psi}$.
 \end{lemma}

\begin{proof}
The proof is given by the following equations.

{{
\centering
\par 
\begin{tabular}{rclr}
  $\widehat{\phi\wedge\psi}$   &=& $\{s \in S \mid s \models \phi\wedge\psi\}  $ &  By def.~of $\widehat{\cdot}$ \\
     & = & $\{s \in S \mid (\mathbb{P}, V,a) \in l(s) \Rightarrow (\mathbb{P}, V,a)\Vdash \phi\wedge\psi\} $ & By def.~of $\models$ \\
      & = & $\{s \in S \mid (\mathbb{P}, V,a) \in l(s) \Rightarrow (\mathbb{P}, V,a)\Vdash \phi \,\, \& \,\, (\mathbb{P}, V,a)\Vdash \psi  \} $ & \\
       &=& $\widehat{\phi} \cap \widehat{\psi}$. & \\
\end{tabular}
}}   

\end{proof}

 \begin{theorem}[Soundness]\label{soundness}
For any conceptual cumulative model $\mathcal{M}$, the consequence relation $\nc_\mathcal{M}$ it defines is a conceptual cumulative relation, i.e., $\nc_\mathcal{M}$ is closed under all the axioms and rules of $\mathbf{CC}$.
\end{theorem}
 \begin{proof}
The proof follows from Lemma \ref{lem:intersection}, analogous to the proof of the soundness of $\mathbf{C}$ w.r.t.~cumulative models given in \cite[Lemma 7]{kraus1990nonmonotonic}.

\end{proof}

We now  show that, given any conceptual cumulative relation $\nc$,  we can  build a conceptual cumulative model $\mathcal{M}$, such that  $\nc_\mathcal{M} = \nc$. Suppose $\nc$ satisfies the axioms and  rules of \textbf{CC}. All definitions will be relative to this relation.

\begin{definition}\label{def:normal}
    A pointed polarity-based model $\mathbb{M}_a$ is said to be normal for a concept $\phi$ if and only if  for all $\psi \in \mathcal{L}$,  $\phi\nc\psi$ implies $\mathbb{M}_a\Vdash\psi$. 
 \end{definition}
Therefore, a pointed polarity-based model is normal for a concept if its pointed object belongs to all of its plausible super-concepts. As relation $\nc$ is reflexive, for any normal pointed polarity-based model $\mathbb{M}_a$ for a concept $\phi$, $\mathbb{M}_a\Vdash\phi$.

\begin{lemma} \label{lem: normal cummulative}
  Let $\nc$ be a cumulative  consequence relation and  $\phi$, $\phi'$ be any concepts.  For all $\phi'$, $\phi \not \nc \phi'$ iff there exists a pointed polarity-based  model $\mathbb{M}_a$ normal for $\phi$, such that   $\mathbb{M}_a \not \Vdash \phi'$.
\end{lemma}
\begin{proof}
    The \textit{if} part follows from Definition \ref{def:normal}. For the converse direction, suppose $\phi \not \nc \phi'$. We shall build a normal pointed polarity-based model for $\phi$ which does not  satisfy $\phi'$. It is enough to show that there exists a pointed polarity-based model $\mathbb{M}_b$ such that 
    $\mathbb{M}_b  \not \Vdash \phi'$, and $\mathbb{M}_b \Vdash \phi''$ for all $\phi \nc \phi''$. 
    Suppose not. Then by compactness, there exists a finite set $D \subseteq \{\phi'' \mid \phi \nc \phi''\}$ such that $\bigwedge D  \vdash \phi'$.  By (CM) and (Cut), we have $\phi \nc \bigwedge D $. By (RW), we get $\phi \nc \phi'$, which is a contradiction. 
\end{proof}

We shall say that $\mathcal{L}$-formulas $\phi$ and $\psi$ are {\em equivalent} and  write $\phi\sim\psi$ if $\phi\nc\psi$ and $\psi\nc\phi$.

\begin{lemma}\label{lem: cummulative equivalence}
    $\phi\sim\psi$ if and only if $\forall\chi\in\mathcal{L}, \phi\nc\chi \Leftrightarrow \psi\nc\chi$. Hence, the relation $\sim$ is an equivalence relation. 
\end{lemma}

\begin{proof}
    The {\em if} part follows from the reflexivity, and the {\em only if} part follows from the following derived rule of $\mathbf{CC}$:

{{
\centering
\par
    $\frac{\phi\nc\psi\quad \psi\nc\phi\quad \phi\nc\chi}{\psi\nc\chi}$\quad\text{\footnotesize (Equivalence)}
    
}}
\noindent The (Equivalence) rule can be derived using the rules (CM), (LLE) and (Cut). 
\end{proof}

We use $\phi/_\sim$ to denote the equivalence class of $\phi$ under $\sim$. 

\begin{definition} \label{def:order cummulative}
   $\phi/_\sim \leq \psi/_\sim$ if and only if there exists $\chi\in\phi/_\sim$ such that $\psi\nc\chi$. 
\end{definition}
From this, we can prove the following lemma analogously to \cite[Lemma 10]{kraus1990nonmonotonic}.

\begin{lemma}
The relation $\leq$ defined above is antisymmetric.
\end{lemma}

\begin{remark}
    Note that the above relation is well-defined, i.e., it does not depend on the choice of the representatives $\phi$ and $\psi$.
\end{remark}


We define a conceptual cumulative model $\mathcal{M}=(S,l,\prec)$ as follows:  $S=\mathcal{L} /_\sim$ is a set of all equivalence classes of concepts under relation $\sim$. $l(\phi/_\sim)= \{\mathbb{M}_a\mid\text{$\mathbb{M}_a$ is a normal model for} \ \phi\}$, and $\phi/_\sim \prec\psi/_\sim $ iff $\phi/_\sim\leq\psi/_\sim$ and $\phi/_\sim \neq\psi/_\sim$. 
It is  easy  to  check that map $l$ is well-defined, 
and that $\prec$ is asymmetric. By Lemma \ref{lem: normal cummulative} and Definition \ref{def:order cummulative}, we can prove the following two lemmas analogously to \cite[Lemma 11]{kraus1990nonmonotonic} and \cite[Lemma 12]{kraus1990nonmonotonic}.

\begin{lemma}\label{minconcept}
For any concept $\phi$, the state $\phi/_\sim$  is the minimum of $\widehat{\phi}$. 
\end{lemma}

\begin{lemma}\label{lem:characterizeofconcequencerelation}
For any concepts $\phi, \psi$, $\phi\nc\psi$ if and only if $\phi\nc_\mathcal{M}\psi$.
\end{lemma}

It is immediate from the above lemmas to get the representation theorem for conceptual cumulative consequence relation as follow.
\begin{theorem}\label{repthm}
    A conceptual consequence relation is a cumulative consequence relation iff it is defined by some conceptual  cumulative model.
\end{theorem}

In fact, the conceptual cumulative model constructed in the above proofs is a strong conceptual cumulative model. Thus, we have proved a stronger result stating that any conceptual cumulative consequence relation  is defined by some strong conceptual  cumulative model. The following corollary follows from Theorem \ref{soundness}, and Theorem \ref{repthm} analogously to the proof of \cite[Corollary 1]{kraus1990nonmonotonic}.
\begin{corollary}
  Let $\mathrm{K}$ be a set of cumulative $\mathcal{L}$-sequents and $\alpha, \beta\in\mathcal{L}$. The following  statements are  equivalent.

\noindent 1.~For any conceptual cumulative model $\mathcal{M}$, $\mathcal{M}\models\mathrm{K}$ implies $\mathcal{M}\models\alpha\nc\beta$.

\noindent 2.~$\alpha\nc\beta$ has a proof from $\mathrm{K}$ in the system $\mathbf{CC}$.  (In  this case, we say  {\em $\mathrm{K}$ cumulatively entails} $\alpha\nc\beta$.
\end{corollary}





The following corollary is immediate from the above corollary. 
\begin{corollary}
     $\mathrm{K}$ entails $\alpha\nc\beta$ if and only if a finite subset of $\mathrm{K}$ does.
\end{corollary}

\subsection{Conceptual Cumulative Ordered Models and Conceptual  Cumulative Reasoning with Loop}

In this section, we introduce conceptual cumulative ordered models and the corresponding reasoning system, conceptual cumulative reasoning with loops $\mathbf{CCL}$. 

\begin{definition}\label{def:conceptual cumulative ordered models}
A conceptual cumulative model $\mathcal{M}=(S,l,\prec)$ is said to be a {\em conceptual cumulative ordered model} if $\prec$ is a strict partial order.  
\end{definition}

Thus, a conceptual cumulative ordered model is a conceptual cumulative model in which the preference relation is asymmetric and transitive. As the rule (Loop) in propositional setting contains no connectives, it can also be seen as a rule for conceptual reasoning. 

{\small
\[\frac{\phi_0\nc\phi_1\quad \phi_1\nc\phi_2\quad\dots\quad\phi_{n-1}\nc\phi_n\quad\phi_n\nc\phi_0}{\phi_0\nc\phi_n}\quad\text{\footnotesize (Loop)}
\]
}
\noindent We call the extension of $\mathbf{CC}$ with the above rule (Loop) $\mathbf{CCL}$ ({\em conceptual cumulative reasoning with loop}). A consequence relation that satisfies all the  rules and axioms of $\mathbf{CCL}$ is said to be {\em loop-cumulative}. Similar to the propositional setting, we can show that the conceptual cumulative ordered models represent the loop-cumulative consequence relations.  The following rule is derivable in $\mathbf{CCL}$  using rules (Loop) and (Equivalence): for any $i,j\in\{0,\dots n\}$,
{\small
\[\frac{\phi_0\nc\phi_1\quad \phi_1\nc\phi_2\quad\dots\quad\phi_{n-1}\nc\phi_n\quad\phi_n\nc\phi_0}{\phi_i\nc\phi_j}.
\]
}

\noindent  The following proposition is proven analogously to the \cite[Lemma 14]{kraus1990nonmonotonic}.

\begin{proposition}
     The rule (Loop) is valid on all conceptual cumulative ordered models.
\end{proposition}

In Section \ref{sec:Example}, we give an example to show that (Loop) is not valid on all the concpetual cumulative models.  In the following parts, we show that the rule (Loop) in fact characterizes conceptual cumulative ordered models.

\subsection{Characterization of Loop-cumulative Consequence Relations}

Given any loop-cumulative relation $\nc$, by Theorem \ref{repthm}, there is a conceptual cumulative model $\mathcal{M}=(S,l,\prec)$ such that $\nc\, = \nc_\mathcal{M}$. In particular, this conceptual cumulative model is the model defined in Section \ref{ssuc:conceptual cumulative models}. Let $\prec^+$ be the transitive closure of $\prec$. We can show that $\prec^+$ is irreflexive and, hence it turns out to be a strict partial order analogous to \cite[Lemma 16]{kraus1990nonmonotonic}. Then, $\mathcal{M}=(S,l,\prec^+)$ is a conceptual cumulative ordered model. 

We get the following proposition from Lemma \ref{minconcept} and the fact that $\prec^+$ is a strict partial order.

\begin{proposition}
  In $\mathcal{M}$, for any $\phi$, the state $\phi/_\sim$ is a minimum of $\widehat{\phi}$. Therefore, $\mathcal{M}$ is a strong cumulative ordered model.
\end{proposition}

The following proposition derives from the above proposition and Lemma \ref{lem: normal cummulative}.

\begin{proposition}
    $\phi\nc\psi$ if and only if $\phi\nc_\mathcal{M}\psi$. 
\end{proposition}
Thus, we get the following representation theorem for loop-cumulative relations.

\begin{theorem}
    A conceptual consequence relation is loop-cumulative if and only if it is defined by some conceptual cumulative ordered model.
\end{theorem}

\subsection{Conceptual Preferential Models}
In this section, we introduce conceptual preferential models and conceptual preferential ordered  models and  show that  they are sound and complete w.r.t.~$\mathbf{CC}$, and~$\mathbf{CCL}$, respectively. 

\begin{definition}\label{def:conceptual  preferential  models}
A conceptual cumulative model $\mathcal{M}=(S,l,\prec)$ is said to be a {\em conceptual preferential model} if the label $l$ assigns a single pointed polarity-based model to each state.   
\end{definition}
In lattice-based logic,  it is possible that $\mathbb{M}, a \Vdash \phi \vee \psi $,  $\mathbb{M}, a \not \Vdash \phi$, and $\mathbb{M}, a \not \Vdash \psi$. Hence, unlike classical preferential models, $\widehat{\phi \vee \psi} =\widehat{\phi} \cup \widehat{\psi}$ is not valid in conceptual preferential models. Thus, the rule (Or) is not valid on conceptual preferential models. In fact, we will show that system $\mathbf{CC}$ is complete w.r.t.~the class of conceptual preferential models. 

\begin{definition}\label{def:supernormal}
 A pointed polarity-based model $\mathbb{M}_a$ is said to be supernormal for a concept $\phi$ if and only if  for all $\psi \in \mathcal{L}$,  $\phi\nc\psi \Leftrightarrow \mathbb{M}_a \Vdash\psi$. 
\end{definition}

\begin{lemma}\label{lem:supernormal}
For any $\phi \in \mathcal{L}$, there exists a pointed polarity-based model $\mathbb{M}_a$ which is supernormal for $\phi$.
\end{lemma}
\begin{proof}
Let $\Gamma =\{\psi \in \mathcal{L} \mid \phi \nc \psi\}$, and  $\Delta =\{\psi \in \mathcal{L} \mid \phi \not\nc \psi\}$. By Lemma \ref{lem: normal cummulative}, for any $\chi \in  \Delta$, there exists a pointed polarity-based model $\mathbb{N}_b$, such that $\mathbb{N}_b \Vdash \psi$ for all $\psi \in \Gamma$, and $\mathbb{N}_b \not \Vdash \chi$. Thus, by Proposition \ref{prop:joint dissatiscfaction}, there exists a  pointed polarity-based model $\mathbb{M}_a$ such that $\mathbb{M}_a \Vdash \psi$ for all $\psi \in \Gamma$, and  $\mathbb{M}_a \not \Vdash \chi$ for all $\chi \in \Delta$. That is,  $\mathbb{M}_a$ is supernormal for $\phi$. 
\end{proof}
Note that this lemma does not hold in the setting of \cite{kraus1990nonmonotonic} as Proposition  \ref{prop:joint dissatiscfaction} does not hold in propositional setting.

For a given defeasible consequence relation satisfying all the rules and axioms of $\mathbf{CC}$, we define a conceptual preferential model similar to the conceptual cumulative model defined in Section \ref{ssec:Characterization  cumulative consequence} with only the following difference: For any state $\phi/_\sim$, $l(\phi/_\sim)=\mathbb{M}_a$, where $\mathbb{M}_a$ is supernormal for $\phi$. It is easy to check that the lemmas \ref{minconcept} and \ref{lem:characterizeofconcequencerelation} hold for preferential model as defined above. Therefore, we get the following result.

\begin{theorem}\label{thm:completeness preferential models}
A conceptual  consequence relation is a cumulative consequence relation iff it is defined by some conceptual preferential model.
\end{theorem}

In the propositional setting, the rule (Loop) is valid on all preferential models. However, this is not the case in conceptual setting (cf.~Section \ref{sec:Example}). A conceptual preferential model $\mathcal{M}=(S,l,\prec)$ is said to be a {\em conceptual preferential ordered model} iff $\prec$ is a strict partial order. Similar to the cumulative models we can show the following theorem.

\begin{theorem}
A conceptual  consequence relation is loop-cumulative iff it is defined by some conceptual  preferential ordered model.
\end{theorem}

\section{Example}\label{sec:Example}
In this section, we give an example to demonstrate  reasoning on conceptual preferential ordered models and conceptual preferential  models.

\vspace{-0.5 cm}
{\small
\begin{table}[]
    \centering
    \begin{tabular}{|c|c|c|c|c|c|}
        \hline
        \textbf{Animal} & &\textbf{Feature} & &\textbf{Concept}
        & \\
        \hline
        $a_1$ & Platypus & $x_1$ & feeds by mammary glands & $C_1$ & mammals\\

        $a_2$ & Tiger & $x_2$ & gives birth to babies & $C_2$ & viviparous animals\\

        $a_3$ & Sparrow & $x_3$ & lays eggs & $C_3$ & oviparous animals\\

        $a_4$ & Scorpion & $x_4$ & is small & $C_4$ & small animals\\

        && $x_5$ & has warm blood & $C_5$ & warm-blooded animals \\
        \hline
    \end{tabular}
    
    \caption{Objects, features and concepts in the polarity-based model}
    \label{tab:objects and features}
    \vspace{-1cm}
\end{table}
}
Let $\mathbb{P}=(A,X,I)$ be a formal context,
where $A =\{a_1, a_2, a_3,a_4\}$ , $X =\{x_1,x_2,x_3,x_4,x_5\}$  are as in Table \ref{tab:objects and features}, and 
$I=\{(a_1,x_1), (a_1,x_3), (a_1,x_5), (a_2,x_1),\\ (a_2,x_2), (a_2,x_5), (a_3, x_3), (a_3,x_4), (a_3,x_5), (a_4, x_2),  (a_4,x_4)\}$. Note that $a I x$ iff animal $a$ has feature $x$. Let $V$ be a valuation which assigns $V(C_1)=(\{a_1,a_2\}, \{x_1,x_5\})$, $V(C_2)=(\{a_2,a_4\}, \{x_2\})$, and  $V(C_3)=(\{a_1, a_3\}, \{x_3,x_5\})$, $V(C_4)=(\{a_3,a_4\},\{x_4\})$, $V(C_5)=(\{a_1,a_2,a_3\},\{x_5\})$, where the concepts $C_1$, $C_2$, $C_3$, $C_4$, and $C_5$ are as in Table \ref{tab:objects and features}. Let $\mathcal{U}$ be the set of pointed polarity-based models $\mathbb{M}_i= (\mathbb{P},V,a_i)$, where $\mathbb{P}$ and $V$ are as defined above and $a_i\in A$. Let $S=\{s_1,s_2,s_3, s_4\}$
and $l(s_i) = \mathbb{M}_i$. Let $\mathcal{M} =(S, l \prec\}$, where $\prec=\{(s_2, s_1), (s_3,s_1), (s_4,s_1)\}$. Intuitively, $\prec$ says that  the Platypus is less typical or common compared to Tiger, Sparrow, and Scorpion, and typicality of other three  is incomparable. 

Note that the model $\mathcal{M}$ defined above is a conceptual preferential ordered model. Since $s_2 \prec s_1$,  $\widehat{C_1} =\{s_1, s_2\}$, and $\widehat{C_2}=\{s_2,s_4\} $, $C_1 \nc_{\mathcal{M}} C_2$, i.e.~typical mammals are viviparous. However, we have $C_1 \wedge C_3 \not \nc_{\mathcal{M}} C_2$. Indeed, typical oviparous mammals are not viviparous. Thus, $\nc_{\mathcal{M}}$ is nonmonotonic. Moreover, note that we have both $C_1 \nc_{\mathcal{M}} C_2$, and $C_2 \nc_{\mathcal{M}} C_2$, but $C_1 \vee C_2=\top   \not \nc_{\mathcal{M}} C_2$ as $\widehat{\top} =\{s_1, s_2,s_3,s_4\}$ has a minimal element $s_3$ which is not in $C_2$. Thus, the rule (Or) is not valid on  $\mathcal{M}$. This shows that unlike the classical case, the rule (Or) need not be valid on conceptual preferential ordered models.

Consider a slight variation of the above scenario, where $S$ and $l$ are the same as in $\mathcal{M}$. However, we have two reasoners $A$ and $B$ with preference orders $\prec_A$ and $\prec_B$. Suppose $\prec_A=\prec \cup \{(s_2, s_4),(s_4, s_3), (s_2,s_3)\}$ (i.e.~$A$ believes Tiger is more typical/common than Scorpion which is more typical/common than Sparrow) and $\prec_B=\prec \cup \{(s_3, s_2)\}$ (i.e.~$B$ believes Sparrow is more typical/common than Tiger). Suppose, $A$ and $B$ want to define a preference relation which both of them can agree on. They come up with the following method to define such preference relation $\prec\,$: for any $s_1$, $s_2$, $s_1 \prec s_2$ iff either (a) $s_1 \prec_B s_2$ or (b) $s_1 \prec_A s_2$ and $s_2 \not \prec_B s_1$, (this can be understood as saying all the preferences of $B$ have to be respected in $\prec$).

In this case, we end up with $\prec=\{(s_4, s_3), (s_2, s_4),(s_3,s_2)\}$ which is non-transitive\footnote{Non-transitive preference orders are common in different fields like psychology, economics, etc.}. Thus, the model $\mathcal{M}' =\{S,l, \prec\}$ is a conceptual preferential model which is not ordered. Note that $C_4\nc_{\mathcal{M}'} C_2$, $C_2 \nc_{\mathcal{M}'} C_5$, and $C_5 \nc_{\mathcal{M}'} C_4$, but $C_4 \not\nc_{\mathcal{M}'} C_5$. Thus, (Loop) is not valid on $\mathcal{M}'$ even though it is a conceptual preferential model.

\section{Conclusion and Future Works}\label{sec:Conclusion and future works}
In this paper, we have taken first steps toward developing non-monotonic reasoning on concepts. We define  generalizations of cumulative reasoning $\mathbf{C}$ and cumulative reasoning  with loop $\mathbf{CL}$ to conceptual setting. We also generalize cumulative models,   cumulative ordered models, and  preferential models  to conceptual setting and show soundness and completeness results for these models.

In the future, we plan to study the defeasible consequence relation defined in terms of typical features, as well as its interaction with the defeasible consequence relation explored in this paper, which is defined in terms of typical objects. Additionally, we aim to generalize defeasible reasoning systems incorporating Rational Monotonicity and ranked preference models \cite{lehmann1992does} to the conceptual setting. Finally, given the close relationship between AGM belief revision and non-monotonic reasoning \cite{gardenfors1994nonmonotonic}, we intend to use these reasoning systems to develop models for belief revision in Formal Concept Analysis. We believe this approach could be particularly beneficial in various applications of FCA, where reasoning about concepts needs to be updated as new knowledge is acquired.

%
%
\bibliographystyle{splncs04}
\bibliography{reference}

\begin{thebibliography}{1}
\providecommand{\url}[1]{\texttt{#1}}
\providecommand{\urlprefix}{URL }
\providecommand{\doi}[1]{https://doi.org/#1}

\bibitem{Blackburn_Rijke_Venema_2001}
Blackburn, P., Rijke, M.d., Venema, Y.: Modal Logic. Cambridge Tracts in Theoretical Computer Science, Cambridge University Press (2001)

\bibitem{BRITZ201155}
Britz, K., Meyer, T., Varzinczak, I.: Preferential reasoning for modal logics. Electronic Notes in Theoretical Computer Science  \textbf{278},  55--69 (2011). \doi{https://doi.org/10.1016/j.entcs.2011.10.006}, \url{https://www.sciencedirect.com/science/article/pii/S1571066111001344}, proceedings of the 7th Workshop on Methods for Modalities (M4M’2011) and the 4th Workshop on Logical Aspects of Multi-Agent Systems (LAMAS’2011)

\bibitem{10.1007/978-3-662-52921-8_10}
Conradie, W., Frittella, S., Palmigiano, A., Piazzai, M., Tzimoulis, A., Wijnberg, N.M.: Categories: How i learned to stop worrying and love two sorts. In: V{\"a}{\"a}n{\"a}nen, J., Hirvonen, {\AA}., de~Queiroz, R. (eds.) Logic, Language, Information, and Computation. pp. 145--164. Springer Berlin Heidelberg, Berlin, Heidelberg (2016)

\bibitem{Conradie_2017}
Conradie, W., Frittella, S., Palmigiano, A., Piazzai, M., Tzimoulis, A., Wijnberg, N.M.: Toward an epistemic-logical theory of categorization. Electronic Proceedings in Theoretical Computer Science  \textbf{251},  167–186 (Jul 2017). \doi{10.4204/eptcs.251.12}

\bibitem{conradie2021nondistributivelogicssemanticsmeaning}
Conradie, W., Palmigiano, A., Robinson, C., Wijnberg, N.: Non-distributive logics: from semantics to meaning (2021), \url{https://arxiv.org/abs/2002.04257}

\bibitem{ganter2012formal}
Ganter, B., Wille, R.: Formal concept analysis: mathematical foundations. Springer Science \& Business Media (2012)

\bibitem{gardenfors1994nonmonotonic}
G{\"a}rdenfors, P., Makinson, D.: Nonmonotonic inference based on expectations. Artificial Intelligence  \textbf{65}(2),  197--245 (1994)

\bibitem{kraus1990nonmonotonic}
Kraus, S., Lehmann, D., Magidor, M.: Nonmonotonic reasoning, preferential models and cumulative logics. Artificial intelligence  \textbf{44}(1-2),  167--207 (1990)

\bibitem{lehmann1992does}
Lehmann, D., Magidor, M.: What does a conditional knowledge base entail? Artificial intelligence  \textbf{55}(1),  1--60 (1992)

\end{thebibliography}
\end{document}